%% file: main.tex
\documentclass[letterpaper, 10 pt, conference, twocolumn]{ieeeconf}  

\IEEEoverridecommandlockouts                              

\input{definitions}

\title{\LARGE \bf
Sliding on Manifolds:\\
Geometric Attitude Control with Quaternions
}

\author{Brett T. Lopez and Jean-Jacques E. Slotine
\thanks{Nonlinear Systems Laboratory, Massachusetts Institute of Technology, Cambridge MA, {\tt\small \{btlopez,jjs\}@mit.edu}}
}

\begin{document}

\maketitle
\thispagestyle{empty}
\pagestyle{empty}


\begin{abstract}
This work proposes a quaternion-based sliding variable that describes exponentially convergent error dynamics for any forward complete desired attitude trajectory.
The proposed sliding variable directly operates on the non-Euclidean space formed by quaternions and explicitly handles the double covering property to enable global attitude tracking when used in feedback.
In-depth analysis of the sliding variable is provided and compared to others in the literature. 
Several feedback controllers including nonlinear PD, robust, and adaptive sliding control are then derived.
Simulation results of a rigid body with uncertain dynamics demonstrate the effectiveness and superiority of the approach.

\end{abstract}


\section{INTRODUCTION}
\label{sec:intro}

Attitude control is a fundamental problem in robotics and aerospace that has a rich history of research.
Several controllers have been proposed that vary in complexity, robustness, and tracking performance.
However, many lack strong performance guarantees in diverse flight conditions -- a facet that is becoming more critical as aerial and space robots find more real-world uses.
Orientation dynamics evolve in a non-Euclidean space which poses challenges for synthesizing attitude controllers since all analysis must be done on the appropriate state space manifold.
One natural manifold is that formed by the special orthogonal group $\mathrm{SO(3)}$ whose elements are all the orthogonal matrices with determinant 1.
Several $\mathrm{SO(3)}$ controllers have been proposed in the literature, most notably \cite{lee2010geometric,fernando2011robust,goodarzi2013geometric,lee2013nonlinear}, but, as shown in \cite{bhat2000topological}, control on $\mathrm{SO(3)}$ can at best achieve \emph{almost global stability}, i.e., only trajectories in an \emph{open set} converge to an equilibrium.
Further, the complexity of working in $\mathrm{SO(3)}$ can make performance and robustness analysis difficult.
This work will show that \emph{global} closed-loop stability is possible by carefully defining quaternion-based feedback terms that capture the underling topology of their non-Euclidean state space. 


An effective method to derive nonlinear controllers is through the use of so-called sliding variables.
Sliding variables often simplify closed-loop stability and tracking analysis by forming a hierarchy of simple reduced order systems \cite{slotine1991applied}.
Once a sliding variable is chosen, several robust and adaptive nonlinear controllers can be immediately employed (see \cite{slotine1991applied}).
However, defining a sliding variable can be difficult when a state space is non-Euclidean.
This is especially true for $\mathrm{SO(3)}$ as the error dynamics of rotation matrices are difficult to analyze and stability is limited to at best almost global.
The aforementioned challenges can be circumvented by instead using quaternions, which live on the three-sphere $\mathbb{S}^3$, to represent orientations.
Several works have synthesized attitude controllers using quaternions \cite{wie1985quaternion,mayhew2011quaternion,fresk2013full,girish2015nonlinear,liu2015quaternion}, and some have proposed quaternion-based sliding variables \cite{lo1995smooth,jan2004minimum,yeh2010sliding,sanchez2013time,zou2017nonlinear}.
However, works that use sliding variables either 1) treat quaternions as if they evolve in Euclidean space and ignore the topology of $\mathbb{S}^3$ or 2) do not address the double covering property where the so-called unwinding phenomenon \cite{mayhew2011quaternion} can lead to longer-than-necessary attitude maneuvers.

\begin{figure}[t!]
    \begin{subfigure}{.31\columnwidth}
         \centering
         \includegraphics[trim=320 10 320 10, clip, width=1\linewidth]{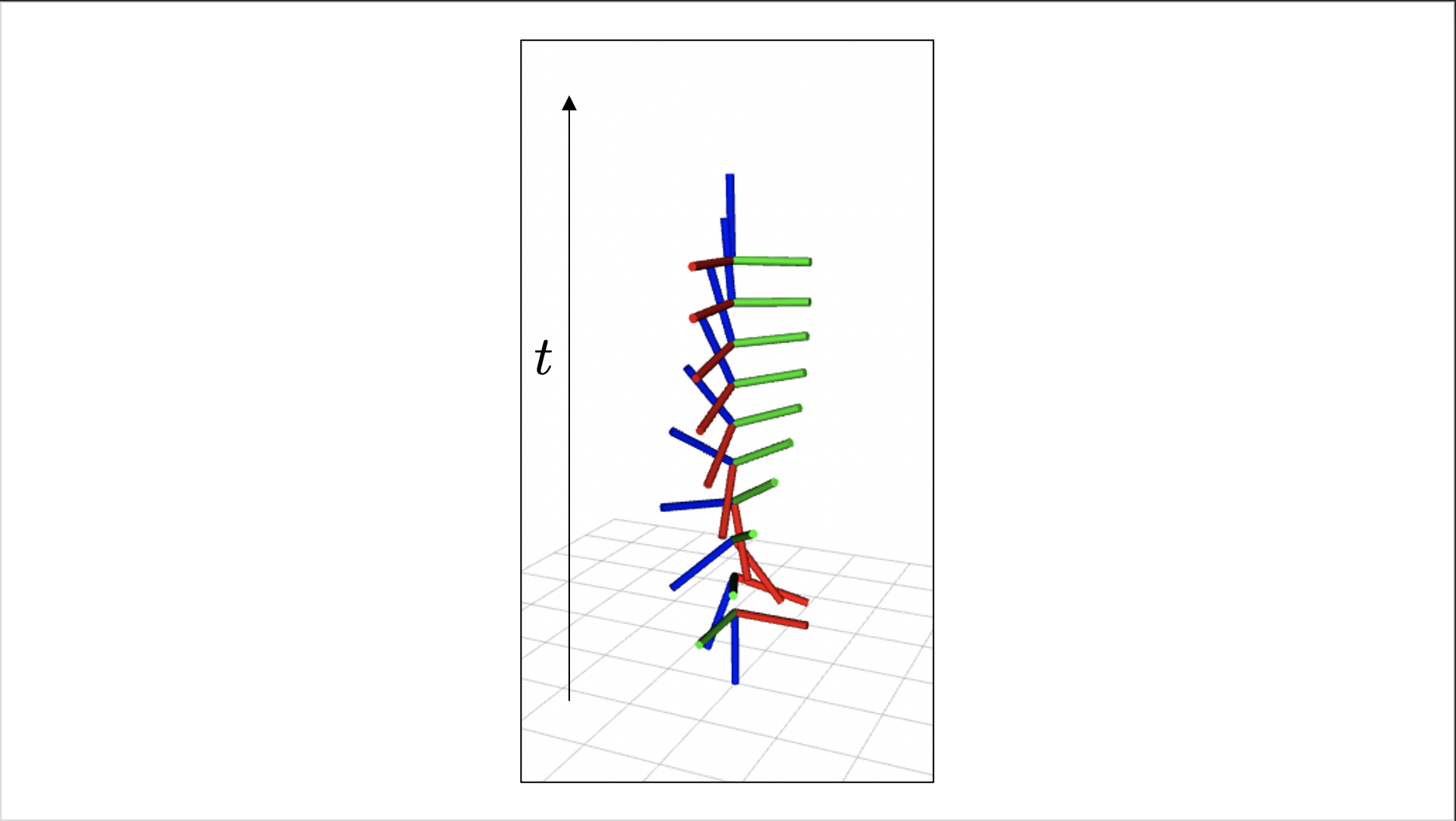}
         \caption{Proposed method.}
         \label{fig:ours}
     \end{subfigure}
    \hspace{0.1em}
     \begin{subfigure}{.31\columnwidth}
         \centering
         \includegraphics[trim=320 10 320 10, clip, width=1\linewidth]{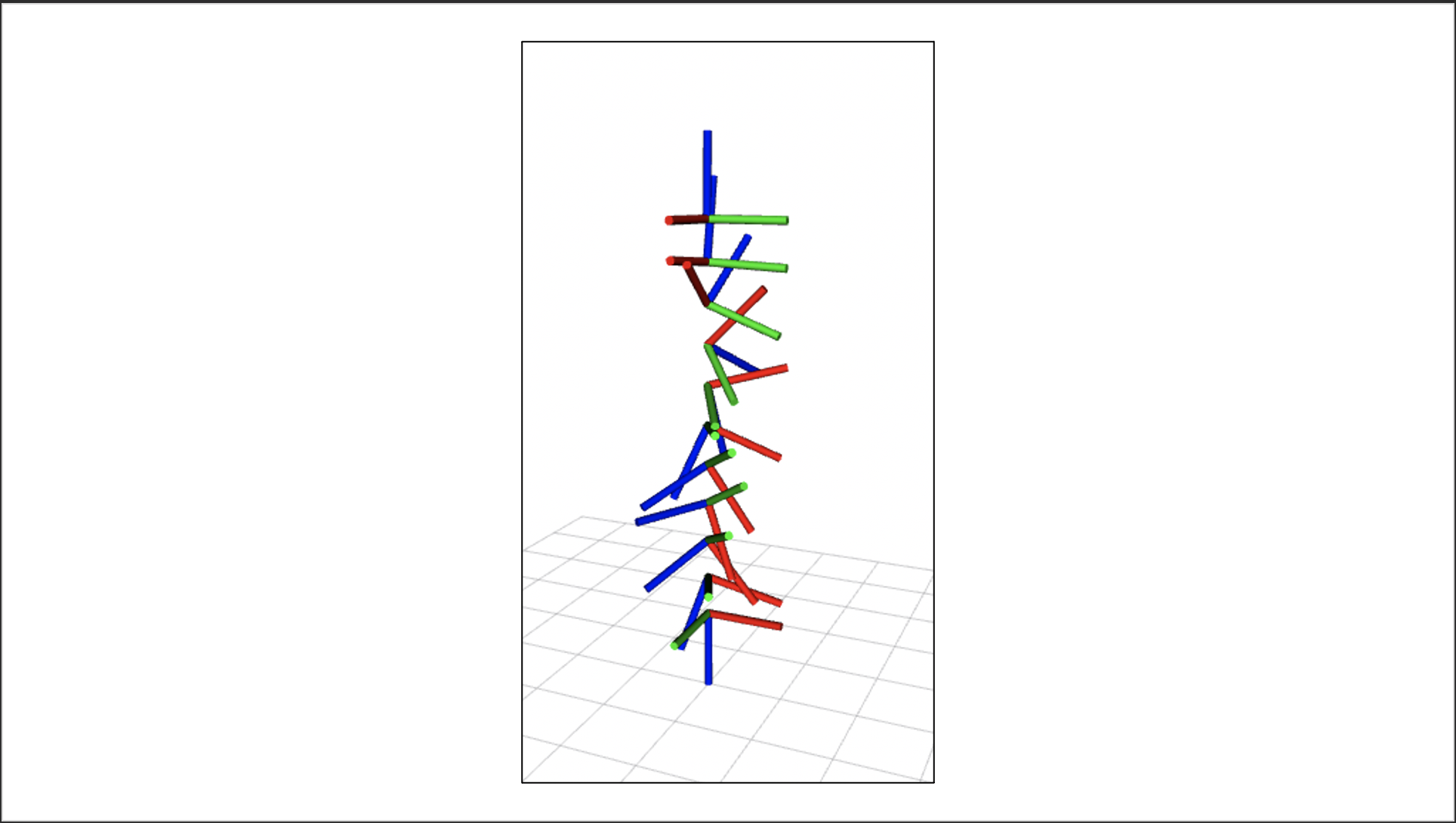}
         \caption{Method from \cite{lo1995smooth}.}
         \label{fig:lo}
     \end{subfigure}
    \hspace{0.1em}
     \begin{subfigure}{.31\columnwidth}
         \centering
         \includegraphics[trim=320 10 320 10, clip, width=1\linewidth]{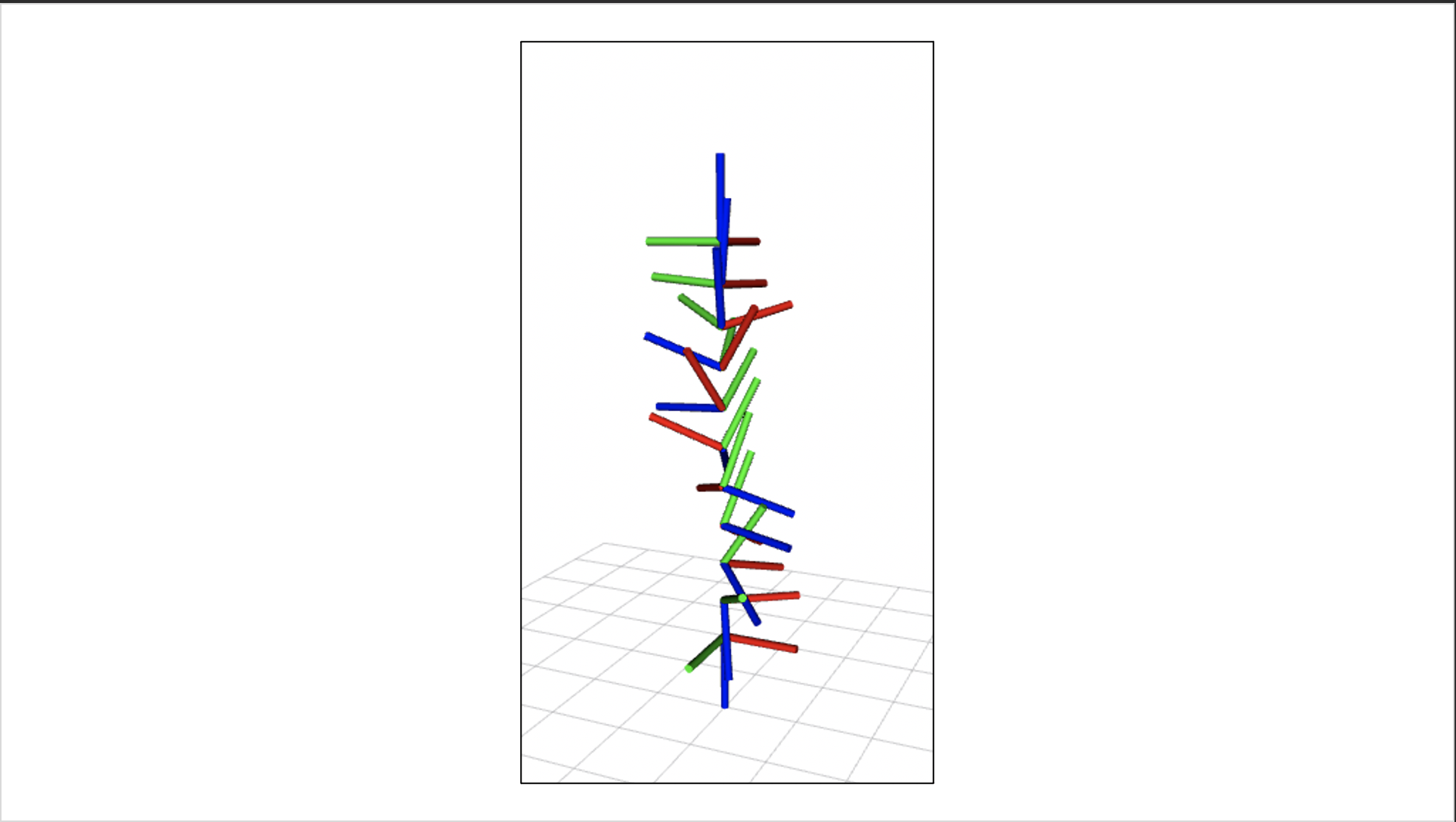}
         \caption{Method from \cite{yeh2010sliding}.}
         \label{fig:yu}
     \end{subfigure}
    \label{fig:flip}
    \caption{Visualization of a rigid body (RGB axes) executing an orientation pointing maneuver during which the orientation quaternion flips sign. Vertical height corresponds to the maneuver duration. The proposed method efficiently executes the pointing maneuver by directly operating on the non-Euclidean space formed by quaternions and explicitly taking into account the double covering property of quaternions.}
    \vskip -0.3in
\end{figure}

The main contribution of this work is a quaternion-based sliding variable that represents exponentially convergent error dynamics.
Exponential convergence is achieved despite explicitly capturing the geometry of the $\mathbb{S}^3$ manifold.
The unwinding behavior is eliminated by introducing a benign discontinuity in the sliding variable that \emph{maintains control input continuity} when used in feedback.
Additionally, the discontinuity results in \emph{global} convergence when the sliding variable is used in feedback.
When compared to other methods the proposed approach has faster convergence without unwinding and is easier to analyze. 
Moreover, the global stability results are much stronger than those obtainable in $\mathrm{SO(3)}$.
Three controllers with strong performance guarantees (nonlinear PD, robust, and adaptive sliding control) are derived using the proposed sliding variable.
The developed method stems from previous work on provably-safe UAV collision avoidance \cite{lopez2018robust} where high-performance attitude control was necessary for establishing safety guarantees despite high model uncertainty. 
This work opens avenues for research in designing trajectories that explicitly consider how a maneuver can exacerbate or reduce the effects of model uncertainty.
Simulation results confirm the predicted performance and demonstrate the benefits of the approach. 


\section{BACKGROUND \& PROBLEM FORMULATION}
\label{sec:formulation}
A unit quaternion $q$ is a four vector that lies on the three-sphere $\mathbb{S}^3$ and consists of a real and vector part, i.e., $q = (\,q^\circ,\vec{q}\,)$. 
The set of unit quaternions form a Lie group under the multiplication operator $\otimes$, where given two quaternions $p$ and $q$ the quaternion product is
\begin{equation*}
    p \otimes q = \left[\begin{array}{c}
         p^\circ q^\circ - \vec{p}^\top \vec{q} \\ p^\circ\vec{q} + q^\circ\vec{p} + p \times q
    \end{array}\right].
\end{equation*}
The inverse of quaternion $q$ is the conjugate quaternion $q^*:=(\,q^\circ,-\vec{q}\,)$ which satisfies $q \otimes q^* = q^* \otimes q = (1,0,0,0)$.
In the context of mechanics, quaternions are particularly useful for representing orientation.
A rotation about a vector $\hat{n}$ through angle $\phi$ can be compactly expressed as $q = (\mathrm{cos}(\nicefrac{\phi}{2}), \mathrm{sin} (\nicefrac{\phi}{2}) \hat{n})$.
Quaternions double cover the special orthogonal group $\mathrm{SO}(3)$ so $q$ and $-q$ represent the same orientation.
There are then two potential ``shortest paths" that connect any two quaternions (see \cref{fig:double_cover}).
The so-called unwinding phenomenon occurs when the longer path (blue curve in \cref{fig:double_cover}) is selected resulting in a longer-than-necessary maneuver.
From a practical and theoretical point of view, any quaternion-based algorithm must capture the geometric properties of the Lie group they form.
Put simply, quaternions must not be treated as vectors in $\mathbb{R}^4$.

The quaternion attitude dynamics can be expressed as
\begin{equation}
    \begin{aligned}
    \dot{q} &= \frac{1}{2}q \otimes \left[\begin{array}{c} 0 \\ \omega\end{array} \right] \\
    J \dot{\omega} &= -\omega \times J\omega  + f(q,\omega) + M_b + d,
    \end{aligned}
    \label{eq:dynamics}
\end{equation}
with $q$ and $\otimes$ are defined as before, angular velocity $\omega \in \mathbb{R}^3$, symmetric positive-definite inertia tensor $J \in \mathscr{S}^3_+$, control input $M_b \in \mathbb{R}^3$, unknown dynamics $f : \mathbb{S}^3 \times \mathbb{R}^3 \rightarrow \mathbb{R}^3$, and unknown disturbance $d\in\mathbb{R}^3$.


The goal of this work is to define a sliding variable that represents exponentially convergent error dynamics while capturing the geometry of the Lie group formed by quaternions. 
By directly considering the topology of $\mathbb{S}^3$ the proposed sliding variable should result in faster tracking error convergence than if they were treated as vectors in Euclidean space. 
Additionally, global convergence results, rather than almost global, should be achieved if the sliding variable is designed carefully.

\section{SLIDING VARIABLES ON MANIFOLDS}
\label{sec:sliding}

\subsection{Overview}
Stability analysis and tracking performance evaluation can be simplified through the use of so-called sliding variables.
If an appropriate sliding variable can be defined then all analysis can be performed on a simpler reduced-ordered system.
For instance, if a hierarchical system is formed by selecting the sliding variable $s$ to describe the desired tracking error dynamics, then the trajectory tracking problem reduces to finding a feedback policy that makes the manifold $\mathcal{S} := \{(\vec{q}_e,\omega_e): s(\vec{q}_e,\omega_e)=0\}$ invariant.
In most cases defining a sliding variable is trivial since many systems operate in $\mathbb{R}^n$.
For systems that lie on a manifold other than $\mathbb{R}^n$, such as $\mathbb{S}^3$ or $\mathrm{SO(3)}$, it can be difficult to construct a sliding variable with simple exponentially convergent  error dynamics.
This section will first present a quaternion-based sliding variable that captures the geometry of $\mathbb{S}^3$ with tracking error dynamics that are globally exponentially convergent.
The proposed quaternion-based sliding variable is then compared to a few common $\mathbb{S}^3\times \mathbb{R}^3$ and $\mathrm{SO(3)} \times \mathbb{R}^3$ alternatives.

\subsection{Sliding Variables in $\mathbb{S}^3\times \mathbb{R}^3$}
Consider the error quaternion $q_e := q_d^* \otimes q$ where $q_d^*$ is the conjugate of the desired quaternion and $q$ is the current quaternion.
The error quaternion has dynamics
\begin{equation}
    \dot{q}_e = \frac{1}{2} q_e \otimes \left[\begin{array}{c} 0 \\ \omega_e\end{array} \right] = \left[ \begin{array}{c} -\vec{q}_e^\top \omega_e \\ q_e^\circ \omega_e + \vec{q}_e \times \omega_e \end{array} \right].
    \label{eq:qe_dynamics}
\end{equation}
Let the sliding variable $s$ be defined as
\begin{equation}
    s = \omega_e + \lambda \mathrm{sgn}_+\left(q_e^\circ\right)\vec{q}_e,
    \label{eq:s}
\end{equation}
with $\omega_e:= \omega - \omega_d$, $\lambda >0$, and 
\begin{equation}
    \mathrm{sgn}_+(\cdot) := \begin{cases}
    1 & \mathrm{if}~\cdot \geq 0 \\
    -1 & \mathrm{if}~\cdot < 0.
    \end{cases}
    \label{eq:sign}
\end{equation}
Note that all operations in \cref{eq:s} are element-wise.
The function $\mathrm{sgn}_+$ in \cref{eq:sign} is different than its standard definition but is necessary for establishing global exponential convergence.
\cref{proposition:qe} shows that if the manifold $\mathcal{S}$ is made invariant via feedback then the the error quaternion $q_e$ converges to the identity quaternion exponentially.

\begin{figure}[t!]
\vskip 0.1in
         \centering
         \includegraphics[scale=0.3]{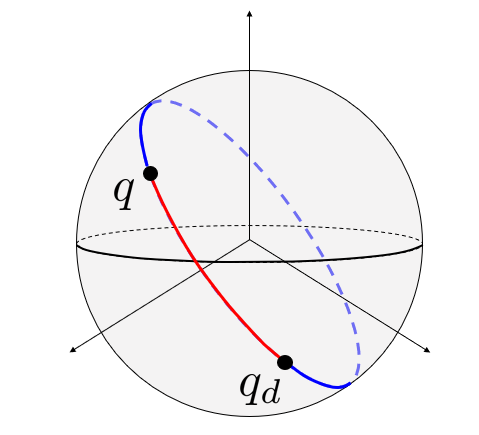}
         \caption{Illustration of the double covering property of quaternions in $\mathbb{S}^2$. Two valid paths (red and blue) connect the actual $q$ and desired quaternion $q_d$. The so-called unwinding phenomenon occurs when a longer path (blue) is taken to reach the desired quaternion $q_d$.}
         \label{fig:double_cover}
    \vskip -0.2in
\end{figure}

\begin{proposition}
    Assume that the manifold $\mathcal{S} = \{(\vec{q}_e,\omega_e): s(\vec{q}_e,\omega_e)=0\}$ is invariant. Then, for any trajectory initialized on $\mathcal{S}$, the quaternion error $q_e$ converges to the identity quaternion exponentially.
    \label{proposition:qe}
\end{proposition}

\begin{proof}
First assume the manifold $\mathcal{S}$ is invariant, i.e., $s=0$ for all $t\geq t_0$.
Then, from \cref{eq:s},
\begin{equation}
    \omega_e + \lambda\mathrm{sgn}_+\left(q_e^\circ\right)\vec{q}_e = 0.
    \label{eq:s0}
\end{equation}
Now consider the candidate Lyapunov function $V = \|\vec{q}_e\|^2$.
Differentiating along the quaternion error dynamics \cref{eq:qe_dynamics},
\begin{equation}
    \dot{V} = 2 \vec{q}_e^\top\dot{\vec{q}}_e = 2 \vec{q}_e^\top \left[ -q_e^\circ \omega_e + \vec{q}_e \times \omega_e \right] = -2 q_e^\circ \vec{q}_e^\top \omega_e,
    \label{eq:v_dot_qe_we}
\end{equation}
where the third equality is obtained by noting the cross product between $\vec{q}_e$ and $\omega_e$ produces a vector perpendicular to $\vec{q}_e$.
Since $s=0$ by assumption, \cref{eq:s0} can be substituted into $\cref{eq:v_dot_qe_we}$ resulting in
\begin{equation}
    \label{eq:v_dot_qe}
    \dot{V} = \frac{d}{dt} \|\vec{q}_e\|^2 = - 2 \lambda |q_e^\circ| \|\vec{q}_e\|^2,
\end{equation}
which shows $\vec{q}_e$ converges exponentially to zero with rate $\lambda |q_e^\circ|$.
When $q_e^\circ=0$, i.e., if $\|\vec{q}_e\|^2 = 1$, from \cref{eq:qe_dynamics} $\dot{q}_e^\circ = - \vec{q}_e^\top \omega_e = \lambda \mathrm{sgn}_+(q_e^\circ) \| \vec{q}_e \|^2$ where \cref{eq:s0} is again utilized.
With the definition of $\mathrm{sgn}_+$ in \cref{eq:sign}, then $\dot{q}_e^\circ > 0$ when $q_e^\circ = 0$ which guarantees that $\frac{d}{dt}\|\vec{q}_e\|^2 = 0$ for all $t\geq t_0$ only when $\|\vec{q}_e\| = 0$.
Hence, $\vec{q}_e$ converges to zero exponentially. 
Furthermore, the error quaternion $q_e$ globally converges to the identity quaternion exponentially.
\end{proof}

\cref{proposition:qe} shows that \emph{global} exponential convergence to the identity quaternion occurs with \cref{eq:s} when $s=0$ indefinitely.
This is a much stronger result than what is obtainable in $\mathrm{SO(3)}$.
Additionally, since the convergence is exponential, if $s\neq0 $ but tends to zero then the error quaternion will tend to the identity quaternion.
This property is key to show convergence with adaptive controllers that utilize sliding variables, and will be discussed further in \cref{sec:adaptive}.

The importance of capturing the topology of $\mathbb{S}^3$ in \cref{eq:s} can be understood by analyzing the convergence of the sliding variable proposed in \cite{lo1995smooth}; the key difference being how the quaternion error is defined.
First let $[\cdot]_\times$ be the operator that maps a vector to a skew symmetric matrix, i.e., $[\cdot]_\times: \mathbb{R}^3 \rightarrow \mathfrak{so}(3)$ where $\mathfrak{so}(3)$ is the Lie Algebra of $\mathrm{SO(3)}$.
Now let $T(q):= q^\circ I + [\,\vec{q}\,]_\times$ where $T(q)$ is invertible except when $q^\circ=0$. 
Define a different sliding variable as
\begin{equation}
\label{eq:s_prime}
    s' = \tilde{\omega} + \lambda \tilde{q},
\end{equation}
where $\tilde{\omega} := \omega - 2T(q)^{-1}\dot{\vec{q}}_d$ and $\tilde{q} := \vec{q} - \vec{q}_d$.
It can be shown that when $s'=0$ then
\begin{equation}
\label{eq:qe_legacy}
    \dot{\tilde{q}} = - \frac{1}{2} \lambda T(q) \tilde{q}.
\end{equation}
For the Lyapunov-like function $V = \|\tilde{q}\|^2$, using \cref{eq:qe_legacy} leads to
\begin{equation}
\label{eq:v_dot_qe_legacy}
    \dot{V} = \frac{d}{dt} \|\tilde{q}\|^2 \leq  - \lambda q^\circ \|\tilde{q}\|^2,
\end{equation}
which shows $\tilde{q} \rightarrow 0$ exponentially for $q^\circ > 0$.
While \cref{eq:v_dot_qe_legacy} is similar to \cref{eq:v_dot_qe}, the convergence behavior is quite different: it takes longer to reach the desired quaternion $q_d$ using the sliding variable \cref{eq:s_prime} than with \cref{eq:s}.
From a differential geometry perspective, this is not surprising since the shortest path between two points in $\mathbb{S}^3$ is known to be a great arc on $\mathbb{S}^3$.
By treating quaternions as vectors in $\mathbb{R}^4$, the vector difference in \cref{eq:s_prime} inherently leads to slower convergence since $\tilde{q} \not\in \mathbb{S}^3$.
By properly defining the quaternion error $q_e$ in \cref{eq:s}, the sliding variable $s$ is guaranteed to remain on the manifold $\mathbb{S}^3 \times \mathbb{R}^3$ resulting in faster convergence.
Hence, the geometric properties of the state space manifold \emph{must be captured} when designing sliding variables if the desired performance is to be achieved. 

Another noteworthy comment on \cref{eq:s} is the definition of $\mathrm{sgn}_+$ in \cref{eq:sign}.
If the standard function $\mathrm{sgn}$ were to be used in \cref{eq:s} in place of \cref{eq:sign} then the sliding surface $\mathcal{S}$ would have two stable equilibria\footnote{As noted previously, quaternions double cover $\mathrm{SO(3)}$ so $q_e$ and $-q_e$ represent the same orientation error.}, mainly $\left(q_e^\circ, \|\vec{q}_e\|\right) = (\pm 1,0)$ and $\left(q_e^\circ, \|\vec{q}_e\|\right) = ( 0,\pm 1)$.
With the choice of \cref{eq:sign}, the second equilibrium is made unstable and $q_e$ exponential convergences to $(\pm 1,0)$. 
Moreover, since $q$ and $-q$ represent the same orientation, exponential convergence is indeed global.
While the discontinuity in \cref{eq:s} may seem disconcerting, the $\mathrm{sgn}_+$ function actually \emph{maintains} continuity of a control input when $s$ is used in feedback.
Furthermore, it ensures the controller selects the \emph{shortest} path to the desired orientation $q_d$ subsequently avoiding the unwinding  phenomenon that can be observed with other quaternion-based controllers (such as \cite{yeh2010sliding, cortes2019sliding}).
This behavior can be deduced by noting if $q_e^\circ < 0$ then the rotation is more the $\pi$ radians.
Taking the conjugate of $q_e$, i.e., negating $\vec{q}_e$, results in a rotation less than $\pi$ radians as desired.
If, for instance, $\mathrm{sgn}_+$ is omitted from \cref{eq:s} then \cref{eq:v_dot_qe} becomes $\dot{V} = -2 \lambda q_e^\circ \|\vec{q}_e\|^2$ which is unstable for $q_e^\circ < 0$ and unwinding occurs.
The above points will be further confirmed with the results presented in \cref{sec:results}.

\subsection{Comparison to Sliding Variables in $\mathrm{SO(3)}\times \mathbb{R}^3$}
The two stated advantages of using the quaternion-based sliding variable \cref{eq:s} are 1) exponential convergence of the error quaternion and 2) the convergence result is global; both of which are straightforward to show.
If a sliding variable is instead defined in $\mathrm{SO(3)}$, then neither of the previous statements are necessarily true let alone easy to prove.
For one, stability analysis with rotation matrices can become fairly difficult.
For another, at best only almost global convergence results are possible.
This section will briefly discuss a recently proposed $\mathrm{SO(3)}$ sliding variable \cite{cortes2019sliding}, which, despite its simple structure, is quite difficult to analyze until it is represented with quaternions.

First let $R \in \mathrm{SO(3)}$ which has kinematics $\dot{R} = R [\omega]_\times$.
If $R_d$ is the desired rotation matrix then the error rotation matrix is $R_e = R_d^\top R$ with kinematics
\begin{equation}
        \dot{R}_e = \left(R_d [\omega_d]_\times \right)^\top R + R_d^\top R [\omega]_\times = R_e [\breve{\omega}_e]_\times,
\end{equation}
where $\breve{\omega}_e := \omega - R_e^\top \omega_d$.
Let $\mathcal{P}(A):=\frac{1}{2}\left(A-A^\top\right)$ be the skew symmetric part of matrix $A$. 
Additionally, let $(\cdot)^\vee$ denote the inverse map of $[\cdot]_\times$.
A candidate $\mathrm{SO(3)}$ sliding variable is then
\begin{equation}
\label{eq:s_R}
    s_R = \breve{\omega}_e + \lambda \left( \mathcal{P}(R_e) \right)^\vee.
\end{equation}
For the candidate Lyapunov function $V_R = \mathrm{trace}\left(3I-R_e\right)$ and under the condition $s_R = 0$ indefinitely, one can show
\begin{equation}
\label{eq:v_R_dot}
    \dot{V}_R = -\frac{\lambda}{2} \|\left(\mathcal{P}(R_e)\right)^{\vee}\|^2 \leq 0,
\end{equation}
which indicates almost global \emph{asymptotic} convergence of $R_e$ to $I$.
This is a much weaker result than \cref{proposition:qe} and would preclude \cref{eq:s_R} from being used for feedback in adaptive control.
However, if one were to convert \cref{eq:v_R_dot} to a quaternion representation, then almost global \emph{exponential} stability can be established in a straightforward manner \cite{culbertson2020decentralized}.
Of course other $\mathrm{SO(3)}$ sliding variables can be defined, e.g., \cite{wang2019geometric}, but are not as simple in structure to \cref{eq:s_R} and are just as complicated -- if not more so -- to analyze.
This not only demonstrates the complexity of defining sliding variables in $\mathrm{SO(3)}$ but also how quaternions can simplify analysis and lead to stronger results.

\section{CONTROL WITH $\mathbb{S}^3\times\mathbb{R}^3$ SLIDING VARIABLES}
\label{sec:control}
\subsection{Overview}
This section shows how the quaternion-based sliding variable presented in \cref{sec:sliding} can be used by several feedback controllers, including robust and adaptive controllers.
The following assumptions are made regarding system \cref{eq:dynamics}.

\begin{assumption}
\label{assumption:d}
The disturbance $d$ belongs to a known closed convex set, i.e., $d \in \mathcal{D} := \{ d: |d| \leq D \}$.
\end{assumption}
\begin{remark}
The operation $|\cdot|$ is element-wise so $D \in \mathbb{R}^3_{\geq0}$.
\end{remark}

\begin{assumption}
\label{assumption:f}
The uncertainty in dynamics $f$ is additive and upper bounded, i.e., $f(q,\omega) = \hat{f}(q,\omega) + \tilde{f}(q,\omega)$ where $\hat{f}$ is the nominal dynamics and $|\tilde{f}(q,\omega)| \leq \mathcal{F}(q,\omega)$ is the uncertain dynamics.
\end{assumption}
\begin{remark}
The operation $|\cdot|$ is element-wise so $\mathcal{F}: \mathbb{S}^3 \times \mathbb{R}^3 \rightarrow \mathbb{R}^3_{\geq0}$. 
\end{remark}

\begin{assumption}
\label{assumption:J}
The uncertainty in inertia tensor is additive and bounded, i.e., $J = \hat{J} + \tilde{J}$ where $\hat{J}$ is the nominal inertia tensor and $|\tilde{J}_{i,j}| \leq \mathcal{J}_{i,j}$ is the uncertain part.
\end{assumption}

\subsection{Nonlinear PD Sliding Control with Feedforward}
\label{sec:nonlinear_pd}
Before considering the more general case of uncertain dynamics, it is instructive to analyze the scenario where only an unknown but bounded disturbance acts on \cref{eq:dynamics}.
\cref{sec:adaptive} will address the scenario where the model is only partially known.
\cref{theorem:pd} shows that with \cref{eq:s} and under Assumption~\ref{assumption:d} a nonlinear PD controller with feedfoward results in global exponential closed-loop stability for forward complete desired trajectories.
\begin{theorem}
\label{theorem:pd}
Consider system \cref{eq:dynamics} with known dynamics $f$ and inertia tensor $J$.
For a bounded disturbance $d$, the tracking error convergences exponentially to a region near the identity quaternion with the control law
\begin{equation}
\label{eq:pd}
    M_b = J \dot{\omega}_d + \omega \times J \omega  + f(q,\omega) - \lambda J \mathrm{sgn}_+(q_e^\circ) \dot{\vec{q}}_e - K s
\end{equation}
where $s$ is the sliding variable \cref{eq:s}, $K = [k_1,k_2,k_3]^\top$, and $Ks$ is a vector with elements $k_i s_i$ for $i=1,2,3$.
\end{theorem}
\begin{proof}
Consider the candidate Lyapunov function $V = s^\top J s$ where $s = \omega_e + \lambda \mathrm{sgn}_+(q_e^\circ) \vec{q}_e$.
Differentiating along the $s$ dynamics,
\begin{align*}
    \dot{V} =&~ 2s^\top J \dot{s} = 2 s^\top \left[J\dot{\omega} - J\dot{\omega}_d + \lambda J \mathrm{sgn}_+(q_e^\circ) \dot{\vec{q}}_e \right] \\
    =&~ 2 s^\top \left[ -\omega \times J \omega + f(q,\omega) + M_b + d \right. \\
    &\left. - J\dot{\omega}_d + \lambda J \mathrm{sgn}_+(q_e^\circ) \dot{\vec{q}}_e \right]
\end{align*}
Using \cref{eq:pd}, 
\begin{equation*}
    \dot{V} = -2 \sum_{i=1}^3 \left( k_i s_i^2  - s_i d_i \right)
\end{equation*}
Since $J$ is symmetric and positive-definite matrix then $\ubar{\alpha}I \leq J \leq \bar{\alpha} I$ for $\ubar{\alpha},~\bar{\alpha} > 0$.
By Assumption~\ref{assumption:d}, each $s_i$ satisfies
\begin{equation*}
    |s_i(t)| \leq \frac{1}{\sqrt{\ubar{\alpha}}} |s_i(0)| e^{-k_i t} + \frac{D_i}{k_i \sqrt{\ubar{\alpha}}} \left(1-e^{-k_i t}\right),
\end{equation*}
which shows $|s_i(t)| \leq \nicefrac{D_i}{k_i \sqrt{\ubar{\alpha}}}$ as $t\rightarrow\infty$ exponentially.
Therefore, since $s$ and the error dynamics form a hierarchy, the error dynamics converge exponentially to region near the identity quaternion.  
\end{proof}

\begin{remark}
\label{remark:continuous}
The $\mathrm{sgn}_+$ term ensures $M_b$ \emph{remains continuous} if the error quaternion $q_e$ were to slip sign during a maneuver.
Smooth control is therefore guaranteed \emph{because} the sliding variable contains a discontinuity.
\end{remark}

\cref{theorem:pd} showed that a nonlinear PD controller with feedforward and the sliding surface defined in \cref{eq:s} results in global exponential closed-loop stability.
The term ``nonlinear PD" is appropriate due to the presence of $\dot{\vec{q}}_e$ in \cref{eq:pd}, which, from \cref{eq:qe_dynamics}, is $\dot{\vec{q}}_e = q_e^\circ \omega_e + \vec{q}_e\times \omega_e$ and arises from the geometry of $\mathbb{S}^3$.
For comparison, a commonly used quaternion-based PD controller without feedforward originally proposed in \cite{wie1985quaternion}, later applied to aerial robots in \cite{fresk2013full}, and used in \cite{cutler2012design} is of the form
\begin{equation*}
    M_b^{pd} = -\mathrm{sgn}(q_e^\circ) K_p \vec{q}_e - K_d \omega,
\end{equation*}
where $K_p,K_d > 0$ and all products are element-wise.
It is straightforward to show global asymptotic closed-loop stability via the LaSalle invariance principle for any $q_d$ and $\omega_d = 0$.
However, little can be said about the performance of the controller for arbitrary $\omega_d$.
In applications where $\omega_d$ is not trivially zero, such as time-varying attitude tracking for spacecraft pointing or even UAV obstacle avoidance, a controller like the one presented in \cref{theorem:pd} is required for better closed-loop performance.

\subsection{Robust Sliding Control}
\label{sec:robust}
The controller presented in \cref{sec:nonlinear_pd} works well when only an unknown but bounded disturbance is present.
As discussed previously, the main advantage of using sliding variables is that several robust and adaptive controllers can be immediately employed.
One such controller is the well-known sliding mode controller which is able to achieve zero tracking error exponentially through high-frequency control switching.
This so-called chattering behavior is synonymous with sliding mode control and makes it unusable for many applications.
This section will instead present a continuous version known as the boundary layer sliding controller.
\cref{theorem:robust} shows how the sliding variable defined in \cref{eq:s} in conjunction with a robust controller gain leads to global exponential closed-loop stability despite model uncertainty.
\begin{theorem}
\label{theorem:robust}
For system \cref{eq:dynamics} with unknown disturbances, dynamics, and inertia tensor that satisfy Assumptions~\ref{assumption:d}-\ref{assumption:J}, the sliding variable $s$ from \cref{eq:s} converges in finite time to the region $|s| \leq \Phi$ with the control law
\begin{equation}
\label{eq:robust_controller}
    M_b = \hat{J} \dot{\omega}_d + \omega \times \hat{J} \omega  + \hat{f}(q,\omega) - \lambda \hat{J} \mathrm{sgn}_+(q_e^\circ) \dot{\vec{q}}_e - K \mathrm{sat}\left(\nicefrac{s}{\Phi}\right),
\end{equation}
where $\mathrm{sat}$ is the saturation function, $\Phi \in \mathbb{R}^3_{>0}$ is the boundary layer thickness, and $K\mathrm{sat}\left(\nicefrac{s}{\Phi}\right)$ is a vector of elements $k_i \mathrm{sat}\left(\nicefrac{s_i}{\Phi_i}\right)$.
The robust gains $k_i$ satisfy
\begin{equation}
\label{eq:K_robust}
    \begin{aligned}
        k_i \geq &~ \left| \omega \times \tilde{J} \omega \right|_i + \left(\mathcal{J} \left| \dot{\omega}_{d} + \lambda \mathrm{sgn}\left(q_e^\circ\right)\dot{\vec{q}}_e\right|\right)_i \\
        & + \mathcal{F}_i + D_i + \eta_i ,
    \end{aligned}
\end{equation}
where $\eta_i > 0$ and $(\cdot)_i$ denotes the $i^{th}$ element of vector $(\cdot)$.
Furthermore, the error quaternion converges exponentially to a region near the identity quaternion.
\end{theorem}
\begin{proof}
Let $s_\Delta:= s - \Phi~ \mathrm{sat}\left(\nicefrac{s}{\Phi}\right)$ and consider the candidate Lyapunov function $V = s_\Delta^\top J s_\Delta$ where $V=0$ when $|s| = \Phi$ and $V > 0$ when $|s|\geq \Phi$.
Differentiating along the $s_\Delta$ dynamics (which are identical to the $s$ dynamics in this case),
\begin{align*}
    \dot{V} =&~ 2s_\Delta^\top J \dot{s}_\Delta = 2 s_\Delta^\top \left[J\dot{\omega} - J\dot{\omega}_d + \lambda J \mathrm{sgn}_+(q_e^\circ) \dot{\vec{q}}_e \right] \\
    =&~ 2 s_\Delta^\top \Big [ -\omega \times J \omega + f(q,\omega) + M_b + d \\
    & - J\dot{\omega}_d + \lambda J \mathrm{sgn}_+(q_e^\circ) \dot{\vec{q}}_e \Big ].
\end{align*}
Using the control law \cref{eq:robust_controller} and noting $\mathrm{sat}\left(\nicefrac{s}{\Phi}\right) = \mathrm{sgn}(s_\Delta)$,
\begin{equation}
\label{eq:v_dot_robust}
\begin{aligned}
    \dot{V} = &~ 2 s_\Delta^\top \Big [ - \omega \times \tilde{J} \omega + \tilde{f} (q,\omega) + d + \tilde{J} \omega_d   \\ 
    &  +\lambda \tilde{J} \mathrm{sgn}_+\left(q_e^\circ\right) \dot{\vec{q}}_e - K \mathrm{sgn}\left(s_\Delta\right)\Big] \\
    = & -2 \sum_{i=1}^{3} k_i|s_\Delta|_i   - 2 s_\Delta^\top \Big [ - \omega \times \tilde{J} \omega + \tilde{f} (q,\omega) \\
    & + d + \tilde{J} \omega_d + \lambda \tilde{J} \mathrm{sgn}_+\left(q_e^\circ\right) \dot{\vec{q}}_e \Big] \\ 
    \leq & - 2 \sum_{i=1}^3 |s_\Delta |_i \Big [ k_i + |\omega \times \tilde{J} \omega|_i \\
    & + \left(\mathcal{J}\left| \dot{\omega}_{d} + \lambda \mathrm{sgn}\left(q_e^\circ\right)\dot{\vec{q}}_e\right|\right)_i  + \mathcal{F}_i + D_i \Big ]\\
    \leq & - 2 \sum_{i=1}^3 \eta_i |s_\Delta|_i,
     \end{aligned}
\end{equation}
where the inequality is obtained by using the robust gain $K$ from \cref{eq:K_robust}.
\cref{eq:v_dot_robust} shows that $s_\Delta$ converges in finite time to zero and, since $s_\Delta = s - \Phi~ \mathrm{sat}\left(\nicefrac{s}{\Phi}\right)$, $s$ converges in finite time to the region $|s| \leq \Phi$.
Therefore, since $s$ and the error quaternion form a hierarchy, the error quaternion globally converges exponentially to a region, defined by the boundary layer thickness $\Phi$, near the identity quaternion.
\end{proof}

\begin{remark}
\cref{remark:continuous} also holds for controller \cref{eq:robust_controller}.
\end{remark}

\begin{remark}
The boundary layer in \cref{theorem:robust} can be made time-varying to capture the effects of state-dependent uncertainty on the tracking error \cite{slotine1991applied}.
A time varying-boundary layer is extremely useful in the context of motion planning and predictive control since they essentially represent a \emph{dynamic robust control invariant tube} which can be used to robustly tighten safety constraints (see, e.g., \cite{lopez2019dynamic,lopez2019adaptive}).
\end{remark}

\cref{theorem:robust} showed that the error quaternion globally converges exponentially to a region near the identity quaternion despite the presence of model error and external disturbances.
This is made possible by making the manifold $\mathcal{S}_{\Phi} = \{ (\vec{q}_e, \omega_e) : |s(\vec{q}_e, \omega_e)| \leq \Phi \}$ invariant with the robust control gain defined in \cref{eq:K_robust}. 
However, the robust gain $K$ can become large for systems with high model uncertainty as it must handle all possible model and disturbance realizations.
Adaptive control is one such method to reduce the robust gain $K$ and lower the control effort from feedback.

\subsection{Adaptive Sliding Control}
\label{sec:adaptive}
The performance of the nonlinear PD controller and the robust controller derived in \cref{sec:nonlinear_pd,sec:robust} can be improved by learning the unknown model parameters.
The nonlinear PD controller can be extended to systems with uncertain dynamics (not handled in \cref{theorem:pd}), while the robust gain $K$ in \cref{theorem:robust} can be reduced. 
To proceed the following standard assumption is made about \cref{eq:dynamics}.
\begin{assumption}
\label{assumption:linear_params}
The dynamics in \cref{eq:dynamics} can be expressed as a linear combination of unknown parameters $a\in\mathbb{R}^p$, i.e., $ J \dot{\omega} = J \dot{\omega}_r + Y(q,\omega,\dot{\omega}_r)^\top a + M_b$ where $Y : \mathbb{S}^3 \times \mathbb{R}^3 \times \mathbb{R}^3 \rightarrow \mathbb{R}^{p \times 3}$ and $\omega_r:=\omega_d - \lambda J {\vec{q}}_e$.
\end{assumption}

Adaptive control for systems like \cref{eq:dynamics} have been extensively studied and a number of approaches exists.
However, only recently have methods been developed that impose physical consistency on parameters, e.g., the inertia tensor being positive-definite, during adaptation \cite{wensing2017linear,lee2018natural}.
This is possible by employing the Bregman divergence operator
\begin{equation}
\label{eq:bregman}
    \mathrm{d}_\psi ( y ~\|~ x) = \psi(y) - \psi(x) - (y-x)^\top \nabla \psi(x),
\end{equation}
where $\psi$ is a strictly convex, continuously differentiable function on a closed convex set. 
The time derivative of the Bregman divergence is
\begin{equation}
\label{eq:bregman_derivative}
    \dot{\mathrm{d}}_\psi (y~\|~x) = (x-y)^\top \nabla^2 \psi(x) \dot{x},
\end{equation}
which will be useful in the context of adaptive control.
Physical consistency is enforced by appropriately selecting $\psi$, such as the log-det function for positive-definite matrices or the log function for bounded parameters.
Another useful property of the Bregman divergence, eloquently shown in \cite{boffi2019higher}, is enforcing sparsity in $\hat{a}$ through appropriate selection of $\psi$.
\cref{theorem:adaptive} shows how the controller in \cref{eq:robust_controller} can be extended to uncertain systems by adding parameter adaptation that enforces physical consistency.

\begin{figure*}[t!]
\vskip 0.05in
    \begin{subfigure}{0.64\columnwidth}
         \centering
         \includegraphics[width=1\textwidth]{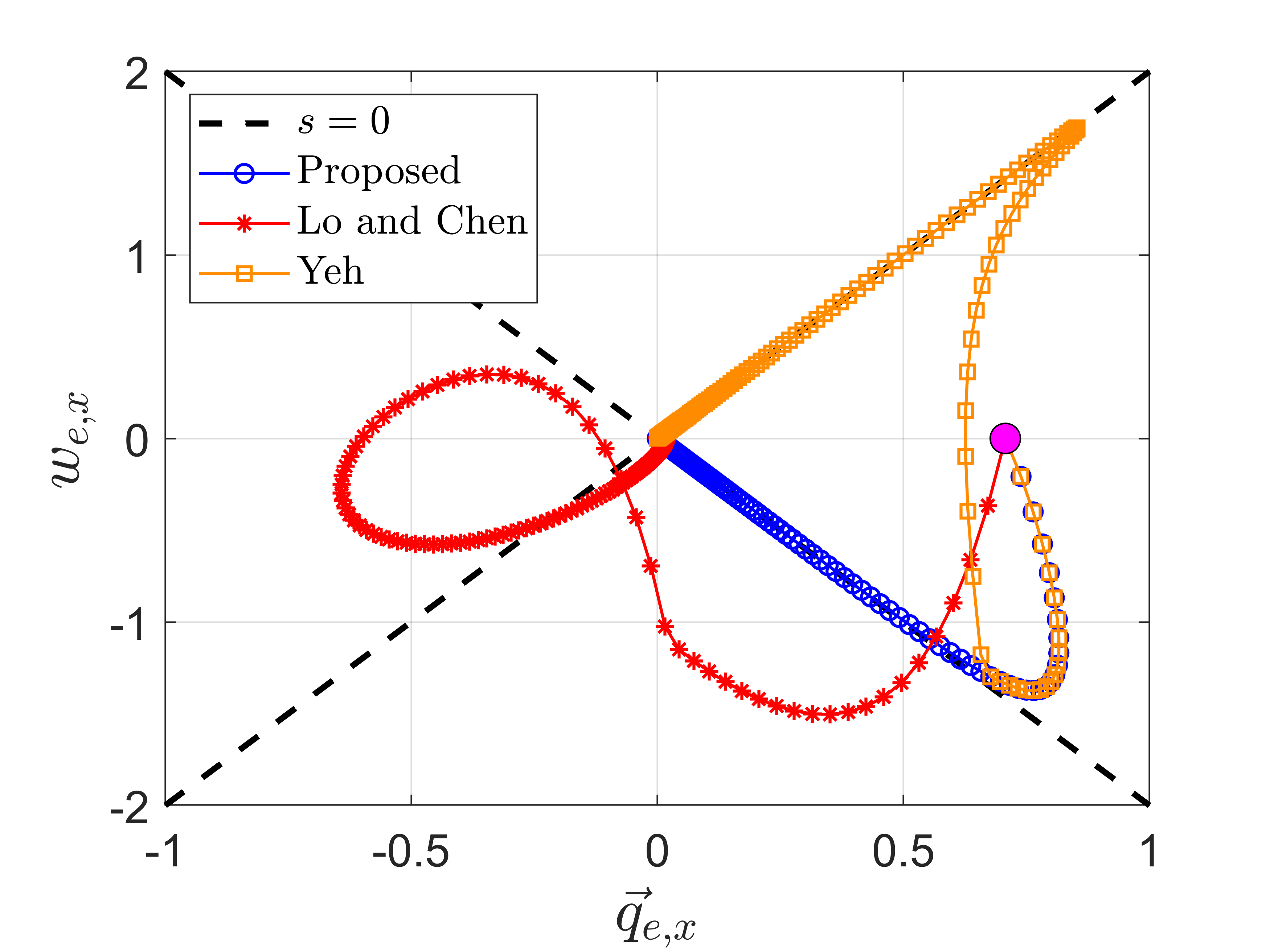}
         \caption{Tracking error along body $x$-axis.}
         \label{fig:sx}
     \end{subfigure}
     \hfill{}
     \begin{subfigure}{0.64\columnwidth}
         \centering
         \includegraphics[width=1\textwidth]{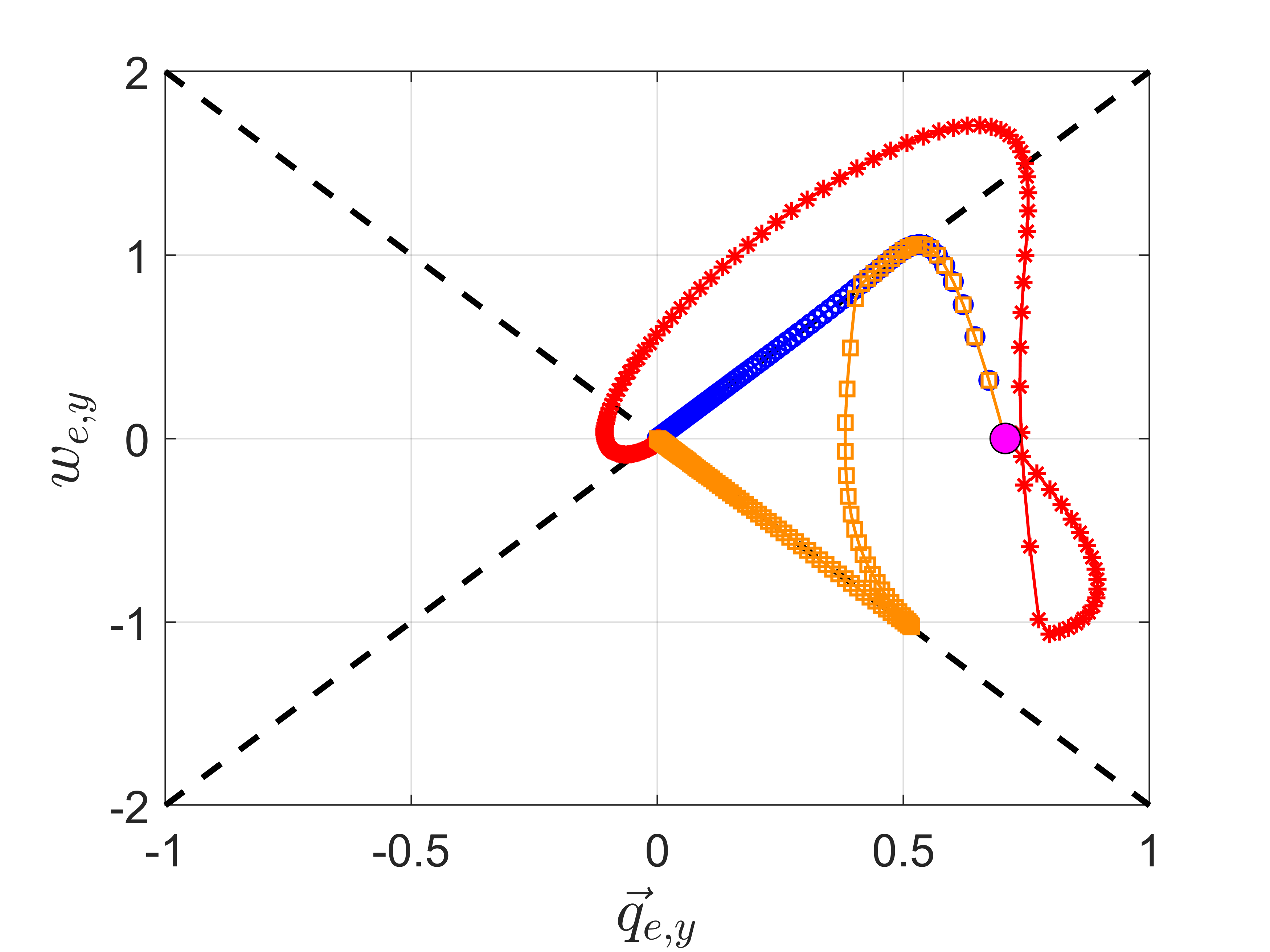}
         \caption{Tracking error along body $y$-axis.}
         \label{fig:sy}
     \end{subfigure}
     \hfill{}
     \begin{subfigure}{0.64\columnwidth}
         \centering
         \includegraphics[width=1\textwidth]{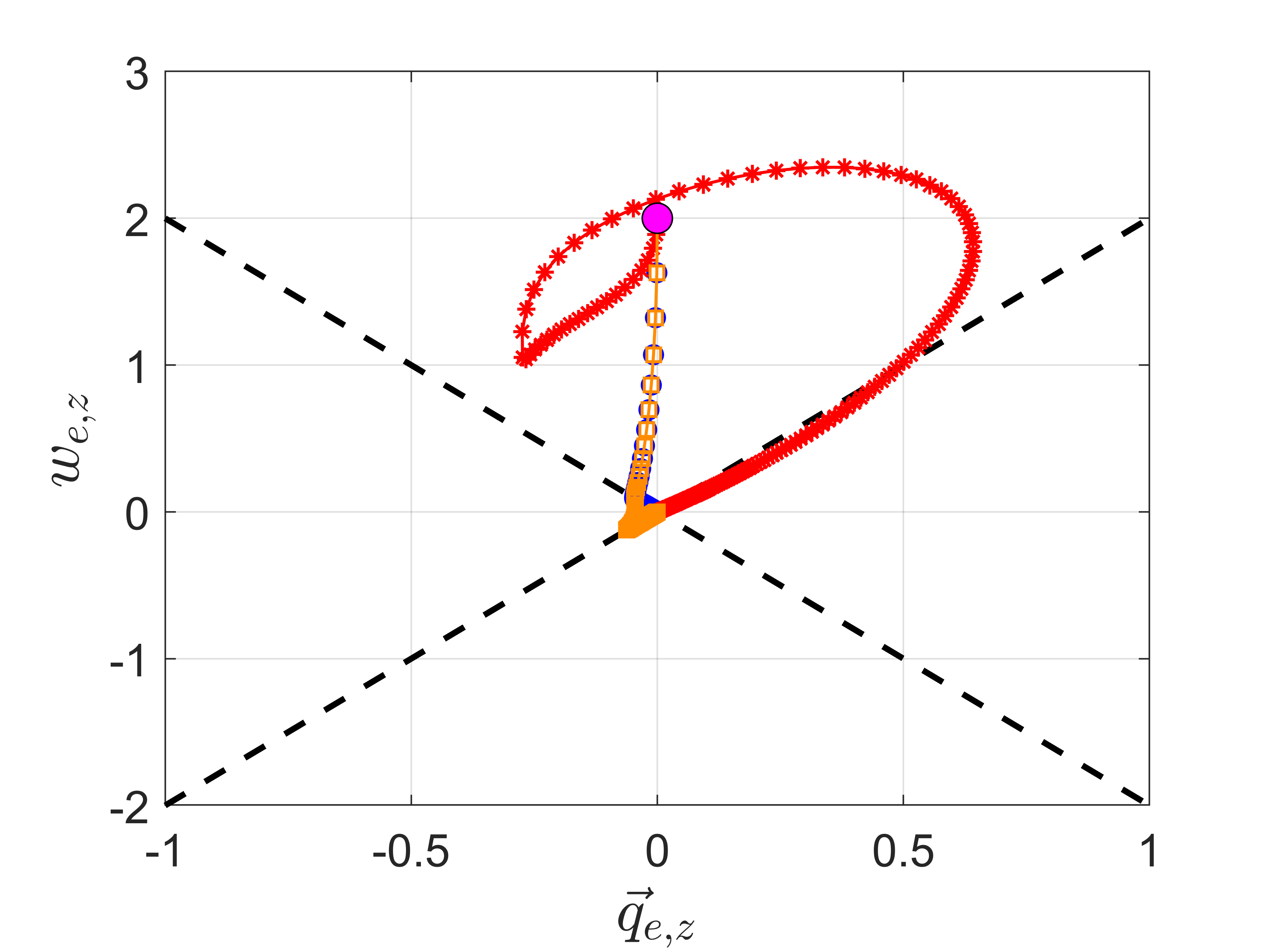}
         \caption{Tracking error along body $z$-axis.}
         \label{fig:sz}
     \end{subfigure}
     \caption{Comparison of tracking error with $q_d = (0.707,0,-0.707,0)$ and $q=(0,1,0,0)$ for a nonlinear PD controller with the proposed sliding variable and those in \cite{lo1995smooth,yeh2010sliding}. The pink circle represents the initial error. The quaternion flips sign midway through maneuver. The tracking error with the proposed sliding variable (blue) converges to one of the two possible sliding manifolds (black). The controller from \cite{lo1995smooth} exhibits undesirable convergence because quaternions are treated as Euclidean vector. The controller in \cite{yeh2010sliding} is similar to the proposed but exhibits unwinding (switches to different sliding manifold).}
    \label{fig:flip_results}
    \vskip -0.2in
\end{figure*}

\begin{theorem}
\label{theorem:adaptive}
Consider system \cref{eq:dynamics} subject to Assumptions~\ref{assumption:J} an \ref{assumption:linear_params}.
The error quaternion asymptotically converges to the identity quaternion with the nonlinear PD controller
\begin{equation}
\label{eq:pd_adpative}
        M_b = Y(q,\omega,\dot{\omega}_r)^\top\hat{a} - K s
\end{equation}
and adaptation law
\begin{equation}
\label{eq:pd_adaptive_law}
        \dot{\hat{a}} = -\left(\nabla^2\psi(\hat{a})\right)^{-1} Y(q,\omega,\dot{\omega}_r) s
\end{equation}
where ${\omega}_r=\omega_d-\lambda\mathrm{sgn}_+(q_e^\circ)\dot{\vec{q}}_e$ and $\psi$ is a continuously differentiable convex function.
\end{theorem}
\begin{proof}
Consider the candidate Lyapunov function 
\begin{equation}
    V =  s^\top J s + 2\mathrm{d}_{\psi}\left(a ~|| ~\hat{a}\right). 
\end{equation}
Differentiating along the $s$ dynamics and utilizing \cref{eq:pd_adpative,eq:bregman_derivative},
\begin{equation*}
\begin{aligned}
    \dot{V} = &~ 2s^\top \left[ Y(q,\omega,\dot{\omega}_r)^\top \tilde{a} - K s  \right] + 2 \tilde{a}^\top \nabla^2 \psi(\hat{a}) \dot{\hat{a}}.
\end{aligned}
\end{equation*}
With the adaptation law \cref{eq:pd_adaptive_law} $\dot{V} = - 2 \sum_{i=1}^3 k_i s_i^2 \leq 0$.
It is easy to verify that $\dot{s}$ is bounded under the condition that the desired angular acceleration $\dot{\omega}_d$ is bounded.
By Barbalat's lemma, $s$ approaches zero asymptotically.
Therefore, by \cref{proposition:qe}, the error quaternion $q_e$ also tends to the identity quaternion as desired.
\end{proof}


\begin{remark}
The adaptation law \cref{eq:pd_adaptive_law} can be immediately used with the robust sliding controller \cref{eq:robust_controller} if adaptation is performed on a subset of unknown parameters.
This is advantageous because the robust gain \cref{eq:K_robust} has fewer uncertain parameters to compensate for resulting in lower control effort.
Proof omitted for brevity but similar to \cref{theorem:adaptive}.
\end{remark}

\section{RESULTS}
\label{sec:results}

The proposed quaternion sliding variable and the controllers presented in \cref{sec:control} were compared to other quaternion-based sliding controllers found in the literature. 
\cref{fig:flip_results} shows the evolution of the tracking error for the nonlinear PD sliding controller (blue) and those presented in \cite{lo1995smooth,yeh2010sliding} (red and orange) when the quaternion $q$ flips sign midway through an orientation pointing maneuver. 
Because of the $\mathrm{sgn}_+$ term in \cref{eq:s}, the sliding surface $\mathcal{S}$ contains \emph{two} manifolds that correspond to the desired error dynamics.
The tracking error with the nonlinear PD controller (blue) exhibits smooth, exponentially convergent  behavior despite the sign flip in $q$.
The controllers proposed in \cite{lo1995smooth,yeh2010sliding} exhibit poor performance as convergence is slow and unwinding (switching between the two  $\mathcal{S}$ manifolds) is observed.
Simulation parameters are $J=\mathrm{diag}(10,10,10)$, $f=[0,0,0]^\top$, $d=[0.2,-0.2,0.2]^\top$, $K=[5,5,5]$, and $\lambda = 2$.

\cref{fig:tracking_error} shows the two-norm of the vector part of the error quaternion when the nonlinear PD, robust, and adaptive controller are used to stabilize \cref{eq:dynamics} with uncertainty in the inertia tensor $ \mathcal{J} = \mathrm{diag}(3,2,4)$.
The boundary layer thickness for \cref{eq:robust_controller} is $\Phi = 0.1$.
The Bregman Divergence with $\psi(x) = \mathrm{log\,det}(x)$ was used to ensure $\hat{J}\succ0$ during model adaptation.
The tracking error with \cref{eq:robust_controller} converges faster than that with \cref{eq:pd_adpative} and \cref{eq:pd} but at the expense of larger control effort, shown in \cref{fig:control_effort}.
Model adaptation can significantly lower control effort while also improving convergence rate when compared to pure nonlinear PD control.


\begin{figure}[t!]
    \begin{subfigure}{.48\columnwidth}
         \centering
         \includegraphics[width=1\linewidth]{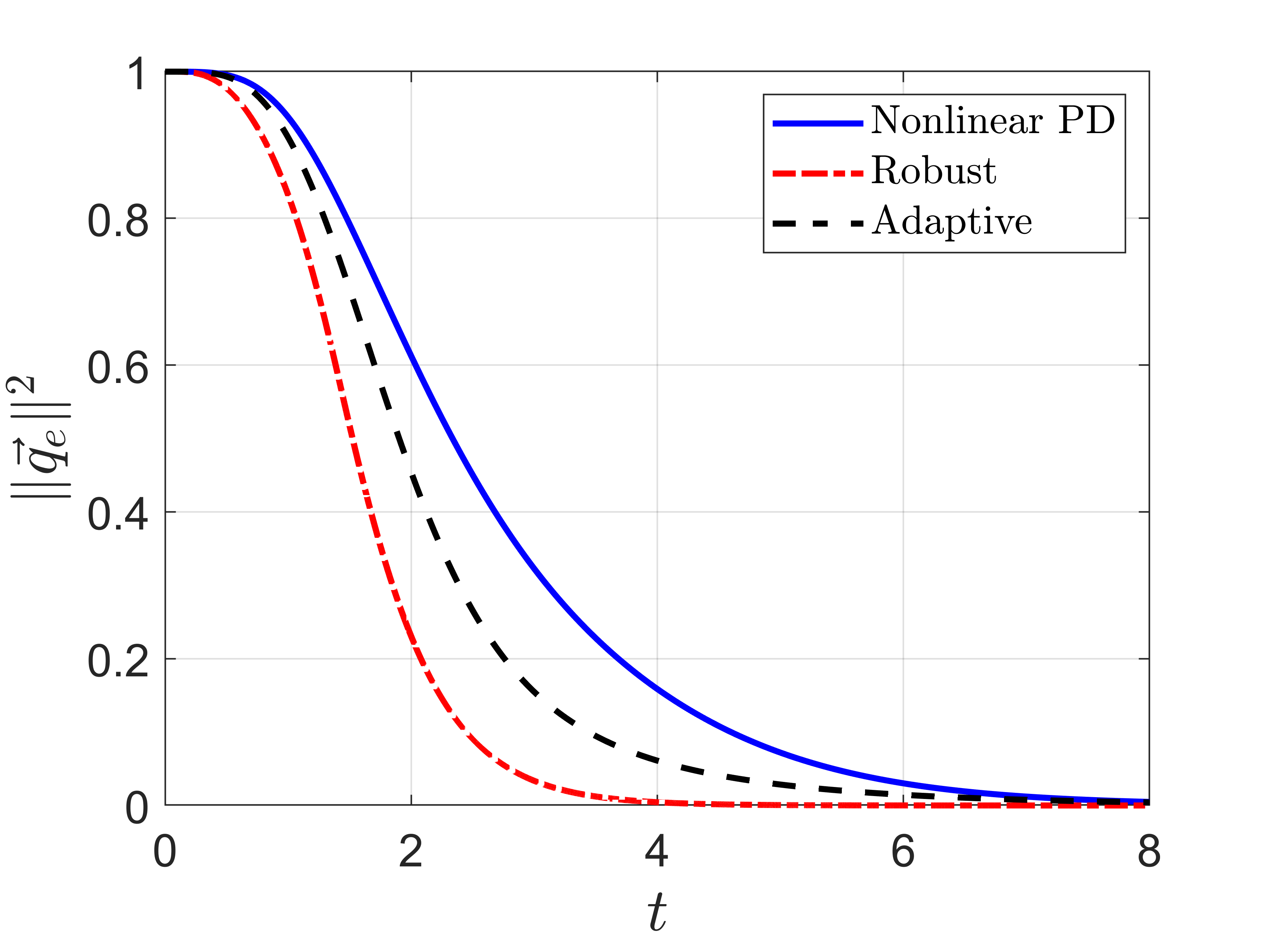}
         \caption{2-norm of $\vec{q}_e$.}
         \label{fig:tracking_error}
     \end{subfigure}
    \hspace{0.1em}
     \begin{subfigure}{.48\columnwidth}
         \centering
         \includegraphics[ width=1\linewidth]{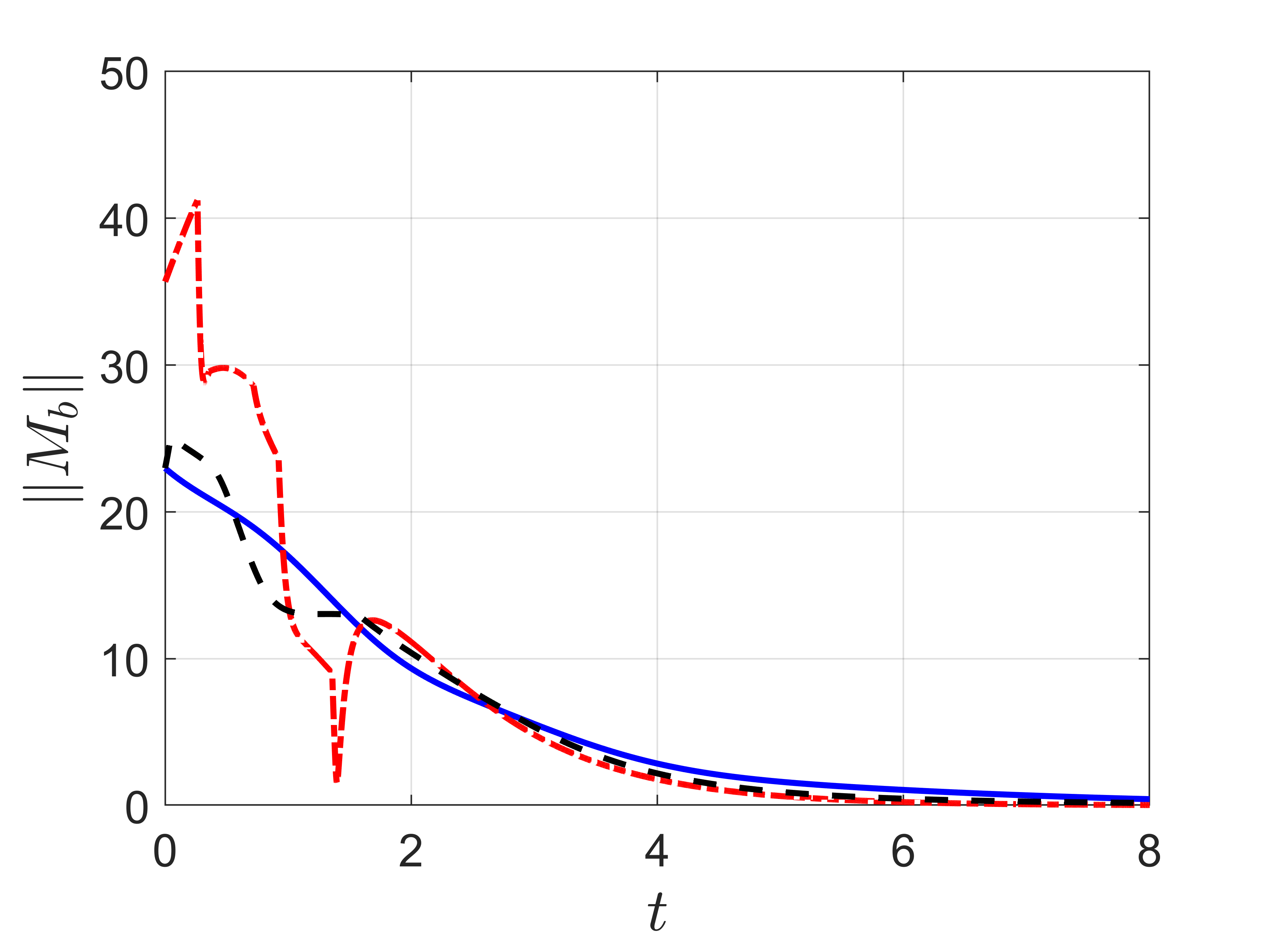}
         \caption{2-norm of control input.}
         \label{fig:control_effort}
     \end{subfigure}
    \hspace{0.1em}
    \label{fig:controller_results}
    \caption{Performance comparison of nonlinear PD, robust, and adaptive sliding control with uncertain inertia tensor. (a): The two norm of $\vec{q}_e$ tends toward zero despite model uncertainty for all controllers but convergence rates vary. Robust control outperforms nonlinear PD and adaptive. Model adaptation improves convergence rate of standalone nonlinear PD control. (b) The two norm of $M_b$ can become large for the robust controller due to model uncertainty compensation. Using an adaptive control approach lowers control effort at the expense of convergence rate.}
    \vskip -0.2in
\end{figure}


\section{DISCUSSION}
\label{sec:conclusion}
This work presented a quaternion-based sliding variable that represents exponentially convergent  quaternion error dynamics.
The two defining characteristics of the proposed sliding variable are that it 1) explicitly captures the topology of $\mathbb{S}^3$ and 2) contains a benign discontinuity -- which actually guarantees the control input is \emph{continuous} -- to eliminate the unwinding behavior.
Although several quaternion and $\mathrm{SO(3)}$ controllers have been proposed, the results presented here are easy to analyze, have guaranteed performance, and are straightforward to implement.
One research avenue of particular interest is the use of the proposed method for provably safe motion planning of aerial robots flying at high-speeds.
Aerodynamic effects can severely degrade the performance of standard attitude controllers making claims on guaranteed safety, i.e., not colliding with obstacles, very difficult to establish.
Moreover, the ability to understand how a planned trajectory impacts performance (through time-varying boundary layers), is an element not yet fully leveraged in planning.
This, and experimental testing on a physical system, is future work.


\balance
\bibliographystyle{ieeetr}
\bibliography{ref}


\end{document}

%% file: definitions.tex
\usepackage[pdftex, pdfstartview={FitV}, pdfpagelayout={TwoColumnLeft},bookmarksopen=true,plainpages = false, colorlinks=true, linkcolor=black, citecolor = black, urlcolor = black,filecolor=black , pagebackref=false,hypertexnames=false, plainpages=false, pdfpagelabels ]{hyperref}

\usepackage[authormarkuptext=name,addedmarkup=bf,authormarkupposition=right]{changes}
\definechangesauthor[name={B.~L.}, color={blue}]{bl}

\makeatletter
\def\set@curr@file#1{%
  \begingroup
    \escapechar\m@ne
    \xdef\@curr@file{\expandafter\string\csname #1\endcsname}%
  \endgroup
}
\def\quote@name#1{"\quote@@name#1\@gobble""}
\def\quote@@name#1"{#1\quote@@name}
\def\unquote@name#1{\quote@@name#1\@gobble"}
\makeatother

\usepackage{graphicx}
\usepackage{epstopdf}
\usepackage{amssymb,amsmath,mathrsfs}

\usepackage{amsthm}
\usepackage[font=footnotesize]{subcaption}

\usepackage[font=footnotesize]{caption}
\usepackage{xcolor}
\usepackage{units}
\usepackage{algorithm}
\usepackage[noend]{algpseudocode}
\usepackage{balance}
\usepackage[sort,compress,noadjust]{cite}
\usepackage{tabularx}
\usepackage{nameref}

\let\labelindent\relax
\usepackage{enumitem}

\theoremstyle{plain}
\newtheorem{theorem}{Theorem}
\theoremstyle{plain}
\newtheorem{proposition}{Proposition}
\theoremstyle{plain}

\theoremstyle{plain}

\theoremstyle{definition}
\newtheorem{assumption}{Assumption}
\theoremstyle{definition}

\theoremstyle{remark}
\newtheorem{remark}{Remark}

\usepackage[hang,flushmargin]{footmisc}

\usepackage[capitalize]{cleveref}
\crefformat{equation}{(#2#1#3)}
\Crefformat{equation}{Equation~(#2#1#3)}
\Crefname{equation}{Equation}{Eqs.}

\usepackage{accents}
\newcommand{\ubar}[1]{\underaccent{\bar}{#1}}